\documentclass{article}

\usepackage[final]{ewrl_2026}

\usepackage[utf8]{inputenc}
\usepackage[T1]{fontenc}
\usepackage{hyperref}
\usepackage{url}
\usepackage{booktabs}
\usepackage{amsfonts}
\usepackage{amsmath}
\usepackage{amssymb}
\usepackage{amsthm}
\usepackage{nicefrac}
\usepackage{microtype}
\usepackage{xcolor}
\usepackage{graphicx}

\newtheorem{proposition}{Proposition}

\title{Memory-State Critic for Asymmetric Actor-Critic\\ with Application to Vision-Based Pursuit-Evasion}

\newcommand{\afl}[1]{\textsuperscript{\rm #1}}
\author{%
  Arthur Louette\afl{1} \quad
  Alejandro S\'anchez Roncero\afl{2} \quad
  Gaspard Lambrechts\afl{3} \quad
  Pascal Leroy\afl{1,4} \\[0.4em]
  \bf Julien Hansen\afl{1} \quad
  Petter \"Ogren\afl{2} \quad
  Damien Ernst\afl{1} \\[0.7em]
  \normalfont\normalsize
  \afl{1}University of Li\`ege, Belgium \quad
  \afl{2}KTH Royal Institute of Technology, Sweden \\
  \afl{3}McGill University \& Mila Qu\'ebec AI Institute, Canada \quad
  \afl{4}Belerion, Belgium \\[0.3em]
  \texttt{arthur.louette@uliege.be}
}

\begin{document}
\maketitle

\begin{abstract}
In partially observable Markov decision processes, the optimal policy generally depends on the history of observations and past actions.
Asymmetric actor-critic methods have become popular to learn such policies when additional information, such as the true state of the environment, is available during training.
The critic, which is not needed at execution, is given access to the state.
A critic conditioned on the state alone is generally ill-defined and yields biased policy gradients.
Conditioning on the state and the history, the history-state critic restores both.
In this paper, we show that conditioning the critic on the state and the policy's own memory, i.e., the internal representation of the history through which the policy selects its actions, is already well-defined and gives unbiased policy gradients, removing the need for a second recurrent approximator of the history.
We call it the \emph{memory-state critic}.
It follows that a critic based on the policy's memory need not backpropagate its loss into that memory, even though the memory is a lossy encoding of the history.
We evaluate the memory-state critic in a vision-based pursuit-evasion environment between two quadrotors across two arena types.
The pursuer is the learning agent, and the evader is sampled per episode from a fixed pool of heuristic behaviours.
The results show that the memory-state critic outperforms the history-state critic and converges faster.
In addition to being unbiased compared to the state-only critic, it maintains a slight edge in the wall arena, where the actor’s history carries information that the privileged state alone does not.
\end{abstract}

\section{Introduction}

In many real-world decision problems, such as robotic control problems, an agent must take actions based on noisy or incomplete observations of the underlying state of its environment.
Such problems are modelled as partially observable Markov decision processes (POMDPs)~\citep{aastrom1965optimal}. 
Reinforcement learning (RL) in POMDPs addresses such decision problems through trial and error to maximise a reward~\citep{sutton1998reinforcement}.
A consequence of this partial observability is that the optimal control strategy at any given time usually depends on the entire history of observations and past actions.
In POMDPs, the history-dependent strategy, or policy, is usually implemented with a recurrent neural network (RNN) that recurrently updates an internal state, called a hidden state, based on successive observations and actions~\citep{ni2022recurrentmodelfreerlstrong}.
The hidden state is trained by backpropagating the RL loss through time, not by predicting states or observations. 
It therefore encodes what the policy needs from the history to act.
Throughout the paper, we employ state for the environment state of the POMDP and \emph{memory} for the policy's recurrent hidden state.

Many successful RL algorithms for POMDPs adopt the asymmetric learning paradigm. 
It takes foundation in the idea is that, in simulation, the true state of the environment is available at training time.
\citet{pinto2018asymmetric} introduced the asymmetric actor-critic algorithm in which the critic accesses the state only during training, while the policy accesses the observation.
This algorithm has been shown to work better in some environments and is widely used in sim-to-real robot learning~\citep{kaufmann2023champion,degrave2022magnetic,vasco2024super}.

\citet{baisero2022unbiased} proved that a critic that conditions only on the state is generally ill-defined in a POMDP.
Indeed, the state alone is not a sufficient statistic of the future execution of the POMDP, i.e., two different histories reaching the same state can induce different action distributions, since the policy conditions on the history.
Instead, the history-state critic, which takes as input both the history and the state, is well-defined because the history and state together are a Markovian variable of the future execution of the POMDP.
In addition, the history-state critic is an unbiased estimate of the history critic, and it can be used in the policy gradient.
In practice, the critic cannot take the full history either.
It has to encode it with an encoder such as a recurrent network of its own, which is trained by the value loss rather than the policy gradient.
Having a critic that takes the history as input has the disadvantage that it needs a second encoder of the history in comparison to  the state-only critic of \citet{pinto2018asymmetric}.
It doubles the recurrent computation at training time and induces a second network.

In this work, we introduce the \emph{memory-state critic}, which conditions on the state and the policy's own memory.
The critic takes as input the policy's memory, but the critic loss is not backpropagated through the recurrent encoder of the history. 
We prove that the memory-state critic gives the same on-policy value function and the same unbiased policy gradient as the history-state critic.
Indeed, the memory is a sufficient statistic of the history for future actions.
In the symmetric setting where the critic does not see the state, reusing the actor's hidden state alone would not give an unbiased gradient.
Indeed, in that setting, the memory of the agent is not a Markovian variable of the future execution of the POMDP, since it is not necessarily a sufficient statistic for future states of the environment.
However, the combination of the state and the actor's memory keeps the policy gradient unbiased.

We evaluate the memory-state critic on a vision-based pursuit-evasion environment between two quadrotors.
The pursuer is the learning agent.
It observes the world through an onboard camera and must catch an evader that regularly leaves the camera frame during maneuvers.
The evader is sampled per episode from a fixed pool of four heuristic behaviours, and this information is part of the critic's privileged state, not of the policy's observation.
The evader has a stationary policy, and the problem reduces to a single-agent POMDP.
We evaluate on two arenas that differ in their degree of partial observability: an open arena with no obstacles and the same arena with a wall to increase the degree of partial observability.
We compare the memory-state critic to three baselines: a state-only critic, a history-state critic, and a reactive critic that conditions on the state and the current observation and action. 
The memory-state critic matches the history-state critic in the final return and converges faster, as expected from the removal of the second recurrent encoder. 
It also marginally outperforms the state-only critic in the wall arena, where the state aliases the history.
We release the environment and our implementation\footnote{\url{https://github.com/LouetteArthur/MemoryStateCritic}}, built on IsaacLab~\citep{mittal2023orbit} and skrl~\citep{Serrano2022skrl}. 

The remainder of the paper is organised as follows. Section~\ref{sec:background} reviews POMDPs and the asymmetric policy gradient. 
Section~\ref{sec:method} introduces the memory-state critic and proves its two properties. 
Section~\ref{sec:experiments} reports the experiments, which are analysed in Section~\ref{sec:results}. 
The paper concludes with Section~\ref{sec:related} discussing the related work and conclusions of our work ~\ref{sec:conclusion}.

\section{Background}
\label{sec:background}

\subsection{POMDPs and history-dependent policies}

A POMDP is a tuple $(\mathcal{S}, \mathcal{A}, \mathcal{O}, T, O, R, P, \gamma)$.
In a state $s_t \in\mathcal{S}$, the agent collects observations $o_t \in \mathcal{O} $ via $O(o_t \mid s_t)$ at time $t$ and selects actions $a_t$ based on the observable history $h_t = (o_0, a_0, \dots, o_t)$.
We denote the set of histories by $\mathcal{H} = \bigcup_{t=0}^\infty \mathcal{H}_t$ with $\mathcal{H}_t \subseteq \mathcal{O} \times (\mathcal{A} \times \mathcal{O})^t$.
After taking action $a_t \in \mathcal{A}$, the state transitions according to $T(s_{t+1} \mid s_t, a_t)$.
The immediate reward is a scalar $r_t = R(s_t, a_t)$, the initial state $s_0$ is drawn from $P(s_0)$.
The discount factor $\gamma \in [0, 1)$ weights the importance of future rewards.
A policy $\pi(a_t \mid h_t)$ maps a history to a distribution over actions, and is said to be optimal when it maximises the expected return $J(\pi) = \mathbb{E}^\pi[\sum_{t=0}^\infty \gamma^t R(s_t, a_t)]$. 
We denote the future return from time $t$ by $G_t = \sum_{k=0}^\infty \gamma^k R(s_{t+k}, a_{t+k})$.

\subsection{Value functions and the asymmetric policy gradient}

The canonical value function in a POMDP is the history value function $V^\pi(h_t) = \mathbb{E}^\pi[G_t \mid h_t]$, with history-action value function $Q^\pi(h_t, a_t) = \mathbb{E}^\pi[G_t \mid h_t, a_t]$.
The history $h_t$ is a sufficient statistic of the future execution under $\pi$, so $V^\pi(h_t)$ is well-defined.
\citet{baisero2022unbiased} introduce the \emph{history-state value function} $V^\pi(s_t, h_t) = \mathbb{E}^\pi[G_t \mid s_t, h_t]$ and the corresponding $Q^\pi(s_t, h_t, a_t) = \mathbb{E}^\pi[G_t \mid s_t, h_t, a_t]$, which additionally condition on the state.
We therefore have $V^\pi(h_t) = \mathbb{E}_{s_t \mid h_t}[V^\pi(s_t, h_t)]$ and, likewise, $Q^\pi(h_t, a_t) = \mathbb{E}_{s_t \mid h_t}[Q^\pi(s_t, h_t, a_t)]$.

In contrast, the state alone is not a sufficient statistic of the process: two histories reaching the same state can induce different action distributions, since the policy conditions on the history. 
As a result, the time-invariant state value function $V^\pi(s)$ is generally ill-defined.
Even when well-defined, the resulting policy gradient is biased~\citep{baisero2022unbiased}.

For a parameterised history-dependent policy $\pi_\theta$, \citet{baisero2022unbiased} showed that the policy gradient can be written equivalently with either value function as a baseline:
\begin{equation}
  \begin{aligned}
      \nabla_\theta J(\pi_\theta)
      &= \mathbb{E}^\pi\!\left[\sum_{t} \gamma^t \nabla_\theta \log \pi_\theta(a_t \mid h_t) \, A^\pi(h_t, a_t)\right]\\
      &= \mathbb{E}^\pi\!\left[\sum_{t} \gamma^t \nabla_\theta \log \pi_\theta(a_t \mid h_t) \, A^\pi(s_t, h_t, a_t)\right],
\end{aligned}
\label{eq:asymmetric-pg}
\end{equation}

where $A^\pi(h, a) = Q^\pi(h, a) - V^\pi(h)$ and $A^\pi(s, h, a) = Q^\pi(s, h, a) - V^\pi(s, h)$. The asymmetric variant on the right uses the extra information $s$ in the critic without changing the policy gradient.

\subsection{Recurrent actor-critic architectures}

In practice, the size of the history $h_t$ grows linearly in $t$ which becomes impractical. The standard approach is to encode it through a recurrent neural network. A function $z_t = f(h_t)$ is a \emph{recurrent encoder} when it admits an update rule $z_{t} = u(z_{t-1}, a_{t-1}, o_t)$ from an initial hidden state $z_{0}$ and an arbitrary null action $a_{0}$. The encoder is deterministic given its inputs, as is standard for recurrent networks. The policy is then parameterised as
\begin{equation}
    \pi_\theta(a_t \mid h_t) = g_\theta(a_t \mid f_\theta(h_t)) = g_\theta(a_t \mid z_t^a),
\label{eq:policy-factorisation}
\end{equation}
where $f_\theta$ is the recurrent encoder and $g_\theta$ is an action head producing the action distribution. We write $z^a_t = f_\theta(h_t)$ for the policy's memory at time $t$.

The history-state critic of \citet{baisero2022unbiased} is implemented analogously, with its own recurrent encoder $z^c_t = f_\psi(h_t)$ on the critic side and an MLP head producing $V_\psi(s_t, z^c_t)$. The two encoders are trained in parallel, $f_\theta$ by the policy gradient and $f_\psi$ by the value loss, typically with TD learning. Figure~\ref{fig:critics} illustrates this design.
It duplicates the recurrent computation at training time, and creates a second representation-learning problem on the critic side. 
The representation $z^c$ learned by the critic may also diverge from $z^a$ learned by the policy, since the two are trained from different losses. 
We show in the next section that this second encoder is not necessary.

\begin{figure}[t]
    \centering
    \includegraphics[width=0.85\linewidth]{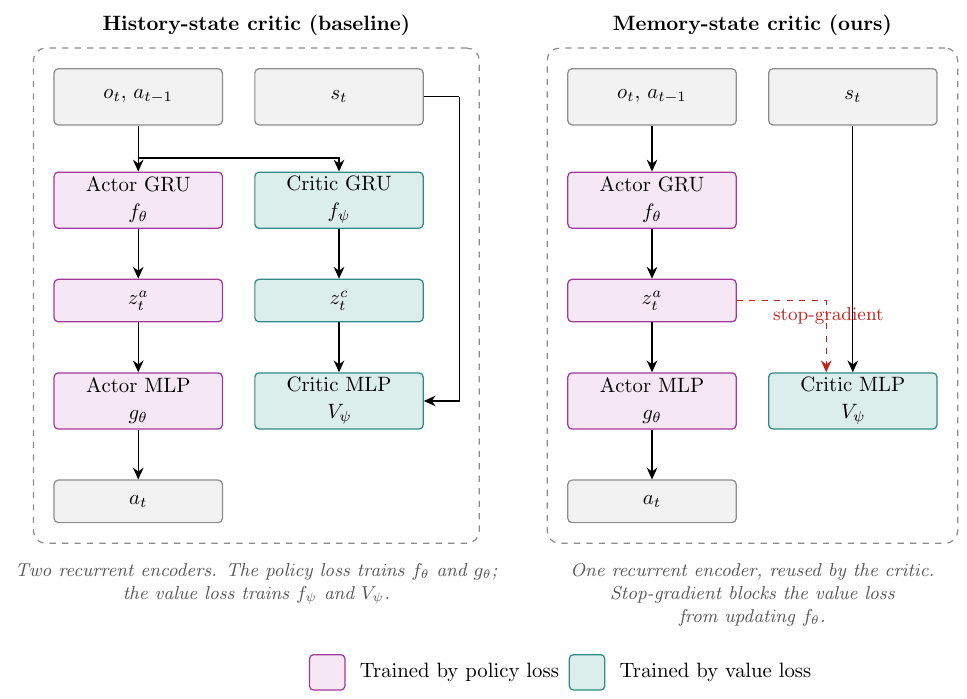}
    \caption{Asymmetric actor-critic architectures with a recurrent policy. Left: the history-state critic maintains a second recurrent encoder $f_\psi$ on the critic side, in parallel with the policy's encoder $f_\theta$. Right: the memory-state critic reuses the policy's hidden state $z^a$ with a stop-gradient, removing the second recurrent module. Both critics consume the state $s$ in their MLP head.}
    \label{fig:critics}
\end{figure}

\section{The memory-state critic}
\label{sec:method}

\subsection{Definition}

Let $z^a_t = f_\theta(h_t)$ be the policy's memory. We define the \emph{memory-state value function} and its action-value counterpart as
\begin{equation}
    V^\pi(s_t, z^a_t) = \mathbb{E}^\pi[G_t \mid s_t, z^a_t], \qquad Q^\pi(s_t, z^a_t, a_t) = \mathbb{E}^\pi[G_t \mid s_t, z^a_t, a_t].
\label{eq:memory-state-value}
\end{equation}
The conditioning event is now the pair $(s_t, z^a_t)$ of the state and the policy's memory, rather than the pair $(s_t, h_t)$ of the state and the full history. We will show that, on-policy, the two value functions agree, and that $V^\pi(s, z^a)$ can be used in the asymmetric policy gradient~(\ref{eq:asymmetric-pg}) without bias. 

\subsection{On-policy equality}

\begin{proposition}[On-policy equality]
\label{prop:equality}
Let $f_\theta$ be a deterministic recurrent encoder, and $\pi_\theta(a \mid h) = g_\theta(a \mid f_\theta(h))$ the corresponding recurrent policy. Then, for all reachable $(h_t, s_t)$,
\begin{equation}
    V^\pi(s_t, z^a_t) = V^\pi(s_t, h_t) \quad \text{and} \quad Q^\pi(s_t, z^a_t, a_t) = Q^\pi(s_t, h_t, a_t) \quad \text{with } z^a_t = f_\theta(h_t).
\end{equation}
\end{proposition}

\begin{proof}
We hereafter show that conditioning on $(s_t, z^a_t)$ and conditioning on $(s_t, h_t)$ induce the same on-policy distribution over future trajectories $(s_{t+k}, a_{t+k})_{k \geq 0}$. 
The policy's action satisfies $a_t \sim g_\theta(\cdot \mid z^a_t)$ by the factorisation~(\ref{eq:policy-factorisation}), so it depends on $h_t$ only through $z^a_t$. The transition $s_{t+1} \sim T(\cdot \mid s_t, a_t)$ depends on $h_t$ only through $a_t$, hence only through $z^a_t$. T
he next observation $o_{t+1} \sim O(\cdot \mid s_{t+1})$ depends on $h_t$ only through $s_{t+1}$, hence only through $(s_t, z^a_t)$. 
The deterministic encoder update $z^a_{t+1} = u(z^a_t, a_t, o_{t+1})$ depends on $h_t$ only through $(z^a_t, a_t, o_{t+1})$, hence only through $(s_t, z^a_t)$. By induction on the trajectory length, the on-policy distributions of $(s_{t+k}, a_{t+k})_{k \geq 0}$ conditioned on $(s_t, z^a_t)$ and on $(s_t, h_t)$ coincide. 
The future-return distributions therefore agree, and the two value functions are equal.
The same induction, started from a fixed $a_t$ instead of $a_t \sim g_\theta(\cdot \mid z^a_t)$, gives the action-value statement.
\end{proof}

\subsection{Asymmetric policy gradient with the memory-state value function}
\begin{proposition}[Asymmetric policy gradient]
\label{prop:pg}
Under the same assumptions as Proposition~\ref{prop:equality}, the asymmetric policy gradient can be written as
\begin{equation}
    \nabla_\theta J(\pi_\theta) = \mathbb{E}^\pi\!\left[\sum_{t} \gamma^t \nabla_\theta \log \pi_\theta(a_t \mid h_t) \, A^\pi(s_t, z^a_t, a_t)\right],
\label{eq:memory-state-pg}
\end{equation}
where $A^\pi(s_t, z^a_t, a_t) = Q^\pi(s_t, z^a_t, a_t) - V^\pi(s_t, z^a_t)$.
\end{proposition}

\begin{proof}
We show that both advantages coincide, so the substitution leaves the gradient unchanged. Step 1 first records that $V^\pi(s,z^a)$ meets the baseline requirement of \citet{baisero2022unbiased}.

\emph{Step 1: $V^\pi(s, z^a)$ averages back to the history value.} \citet{baisero2022unbiased} prove that $V^\pi(h_t) = \mathbb{E}_{s_t \mid h_t}[V^\pi(s_t, h_t)]$. Substituting Proposition~\ref{prop:equality},
\begin{equation}
    \mathbb{E}_{s_t \mid h_t}[V^\pi(s_t, z^a_t)] = \mathbb{E}_{s_t \mid h_t}[V^\pi(s_t, h_t)] = V^\pi(h_t).
\end{equation}
The memory-state value function is therefore an unbiased estimate of the history value function.

\emph{Step 2: the two advantages coincide.} Proposition~\ref{prop:equality}, applied also to the action-value function, makes the two advantages coincide pointwise on every reachable $(s_t, h_t, a_t)$:
\begin{equation}
    A^\pi(s_t, z^a_t, a_t) = Q^\pi(s_t, z^a_t, a_t) - V^\pi(s_t, z^a_t) = Q^\pi(s_t, h_t, a_t) - V^\pi(s_t, h_t) = A^\pi(s_t, h_t, a_t).
\end{equation}
The two expectations are therefore taken of the same random variable, and~(\ref{eq:memory-state-pg}) follows.
\end{proof}

\paragraph{The state is essential.}
Both propositions rely on having the state $s_t$ in the conditioning.
Without it, the policy's memory $z^a_t$ alone is not a sufficient statistic for the future return.
Two histories that map to the same $z^a$ can correspond to different states, and the state in turn drives the future transitions through $T(\cdot \mid s_t, a_t)$.
\citet{lambrechts2025theoreticaljustificationasymmetricactorcritic} gives a finite-time convergence analysis of asymmetric critics with linear function approximators, in which the state appears as exactly the quantity that eliminates aliasing in the agent state.
Proposition~\ref{prop:equality} is the analogous statement, the state closes the aliasing left by the memory.

\subsection{Approximation in practice}

The two value functions $V^\pi(s_t, h_t)$ and $V^\pi(s_t, z^a_t)$ are analytical objects. In practice, both are approximated by neural networks trained with TD learning.
The memory-state critic uses the same architecture for the value head but reads its recurrent input from the policy's encoder:
\begin{equation}
    V_\psi(s_t, z^a_t) \quad \text{with } z^a_t = f_\theta(h_t).
\end{equation}
The critic shares $f_\theta$ in the forward pass. The critic loss is not backpropagated through $f_\theta$. The encoder $f_\theta$ is updated only by the policy gradient. 
A stop-gradient placed on $z^a_t$ in the input to the critic implements this.
Without the stop-gradient, the value loss would flow into the policy's encoder and shape $z^a$ to minimise value error alongside the policy loss, which is known to destabilise on-policy training \citep{cobbe2021ppg}. 
The resulting architecture is shown on the right of Figure~\ref{fig:critics}.
By comparison, the history-state critic of \citet{baisero2022unbiased} is approximated by $V_\psi(s_t, z^c_t)$ with $z^c_t = f_\psi(h_t)$, where $f_\psi$ is a critic-side recurrent encoder with parameters $\psi$ disjoint from $\theta$. Both $f_\theta$ and $f_\psi$ are trained, the first by the policy gradient and the second by the value loss.

\section{Experiments}
\label{sec:experiments}

\subsection{Environment}

We evaluate the memory-state critic on a vision-based pursuit-evasion environment between two Crazyflie quadrotors. 
The pursuer is the only learning agent. 
It must get within a 20 cm radius of the evader and within 10 seconds to neutralise it.
Otherwise, the pursuer gets penalised.
The arena is a $5 \times 4 \times 2$ metre box.
We use two geometries that differ in their degree of partial observability. 
The \emph{open arena} has no obstacles, so the evader stays visible whenever it is inside the camera frame.
The \emph{wall arena} contains an opaque vertical wall with a gap, which can hide the evader from the pursuer even when the camera points at it. 
The wall arena is the environment where the theory predicts the largest gap between critics, because occlusion creates histories that the privileged state alone cannot disambiguate.

The pursuer observes a $64 \times 64$ two-channel image from a camera, with a depth channel and a binary opponent-segmentation mask, together with its previous action.
A visualisation of the pursuer input is presented in Figure \ref{fig:env}.
The policy is therefore conditioned on the history.
The critic additionally sees a 64-dimensional privileged state $s_t$ with the state of both drones, relative-pose features, propeller speeds, and an identifier of the heuristic drawn from the evader pool at the start of the episode.
The pursuer acts through a four-dimensional body-rate and collective-thrust command at $25$ Hz, tracked by an inner PD loop at $500$ Hz.
Dynamics are modelled at the propeller level. Implementation details are in Appendix~\ref{app:appendix}.

The evader is sampled per episode from a fixed pool of four heuristic behaviours: a stationary hover, a circular trajectory, a lemniscate trajectory, and an artificial-potential-field reactive policy\citep{maarif2021afp}. Each heuristic acts as a function of the state, and its identifier at episode start is part of $s_t$. From the pursuer's point of view, the evader is therefore part of the environment, and the setting reduces to a single-agent POMDP that satisfies the assumptions of Proposition~\ref{prop:equality}.
The construction does not directly apply when the opponent is a learned recurrent policy that conditions on its own history and does not take into account that the other might still be learning.
We leave it to future work. 

\begin{figure}[t]
    \centering
    \includegraphics[width=0.95\linewidth]{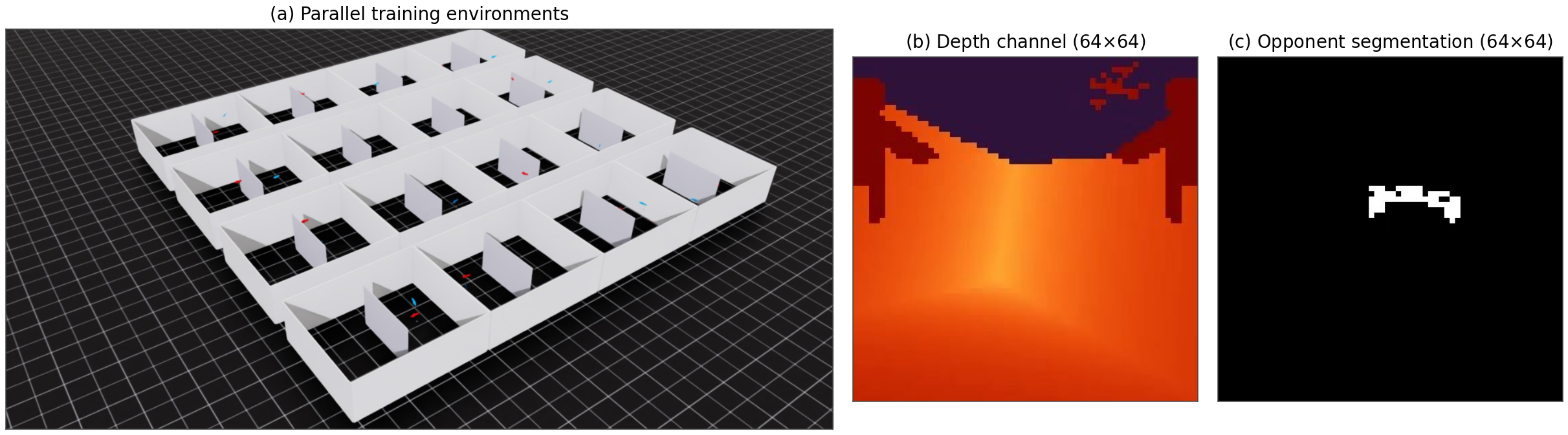}
    \caption{Vision-based pursuit-evasion environment. Left: IsaacLab simulates many arenas in parallel. Each contains a pursuer (red), an evader (blue), and an optional wall with a gap that can hide the evader. Middle: a representative depth channel from the pursuer's camera. Right: the binary opponent segmentation mask. Both channels are $64 \times 64$.}
    \label{fig:env}
\end{figure}

Let $d_t$ be the inter-agent distance, replaced by a wall-aware distance $\tilde d_t$ in the wall arena.
Let $\delta^+_t$ be the capture flag, $b^P_t$ and $w^P_t$ the pursuer out-of-bounds and wall-collision flags, $v_t$ a visibility flag that equals one only when the evader is visible, and $\rho_t$ the angular offset of the evader from the camera axis.
The pursuer's reward is
\begin{equation}
r^P_t = \underbrace{\kappa_a \big[\Phi(\tilde d_t) - \Phi(\tilde d_{t-1})\big]}_{\text{approach}} + \underbrace{\kappa_p \, v_t \, e^{-\kappa_\rho \rho_t}}_{\text{visibility}} - \underbrace{ 1/T + R \, \delta^+_t}_{\text{time and capture}} - \underbrace{ R \, (b^P_t + w^P_t)}_{\text{crash}} - \kappa_{\text{br}} \|W_{\text{br}} \omega_t\|,
\label{eq:reward}
\end{equation}
with approach potential $\Phi(d) = \exp(-d/d_0)$. 
The visibility term rewards keeping the evader near the camera axis, but only when it is visible, so the agent cannot earn it by pointing at an occluded evader.
Coefficient values are in Appendix~\ref{app:appendix}.

\subsection{Architecture and training}

We integrate the memory-state critic into proximal policy optimization\citep{schulman2017ppo}, extended for recurrent actors, building on skrl\citep{Serrano2022skrl}.
The actor is a convolutional encoder followed by a GRU and an action head.
It consumes the onboard camera observation and the previous action, and outputs a continuous action distribution.
The critic is a multi-layer perceptron that takes the privileged state $s_t$ concatenated with the recurrent encoded history chosen by the critic variant.
For the memory-state critic, the recurrent encoded history is the actor's hidden state $z^a_t$. 
For the history-state critic, it is a critic-side encoder $f_\psi$ with independent weights that generates the hidden state $z^c_t$. 
Both encoders share the same architecture.
For the wall arena, we use a slightly larger discount factor $\gamma = 0.999$ to provide a longer credit-assignment horizon around the wall gap.
All other hyperparameters are unchanged.
Hyperparameters are in Appendix~\ref{app:appendix}.

\subsection{Critic variants and protocol}

We compare four critics, listed in Table~\ref{tab:critics}. They differ only in the input of the critic. The actor architecture and all other hyperparameters are identical across variants. 
The observation-state critic tests whether a recurrent encoder on the critic side is necessary at all or if observation and previous action are enough.
The state-only and history-state critics are the baselines from prior work.
We train each critic variant with proximal policy optimization for $5.12 \times 10^7$ environment steps under the same hyperparameters (Appendix~\ref{app:appendix}). 
We run 20 seeds per variant for $V_\psi(s)$, $V_\psi(s, o, a)$, $V_\psi(s, z^c)$, and $V_\psi(s, z^a)$ in the open and wall arenas.
We report aggregate performance using the interquartile mean (IQM) across seeds \citet{agarwal2022deepreinforcementlearningedge}. Bands are 95$\%$ bootstrap confidence intervals over seeds.

\begin{table}[h]
\centering
\caption{Critic variants compared in the experiments.}
\label{tab:critics}
\begin{tabular}{lll}
\toprule
Critic & Symbol & Description \\
\midrule
Observation-state & $V_\psi(s, o, a)$ & State, current observation, previous action. \\
State-only & $V_\psi(s)$ & State only. Biased with a policy conditioned on the history. \\
History-state & $V_\psi(s, z^c)$ & State and a critic-side recurrent encoding of the history. \\
Memory-state (ours) & $V_\psi(s, z^a)$ & State and the policy's memory with stop-gradient. \\
\bottomrule
\end{tabular}
\end{table}
\section{Results}
\label{sec:results}
\begin{figure}[h]
    \centering
    \includegraphics[width=0.95\linewidth]{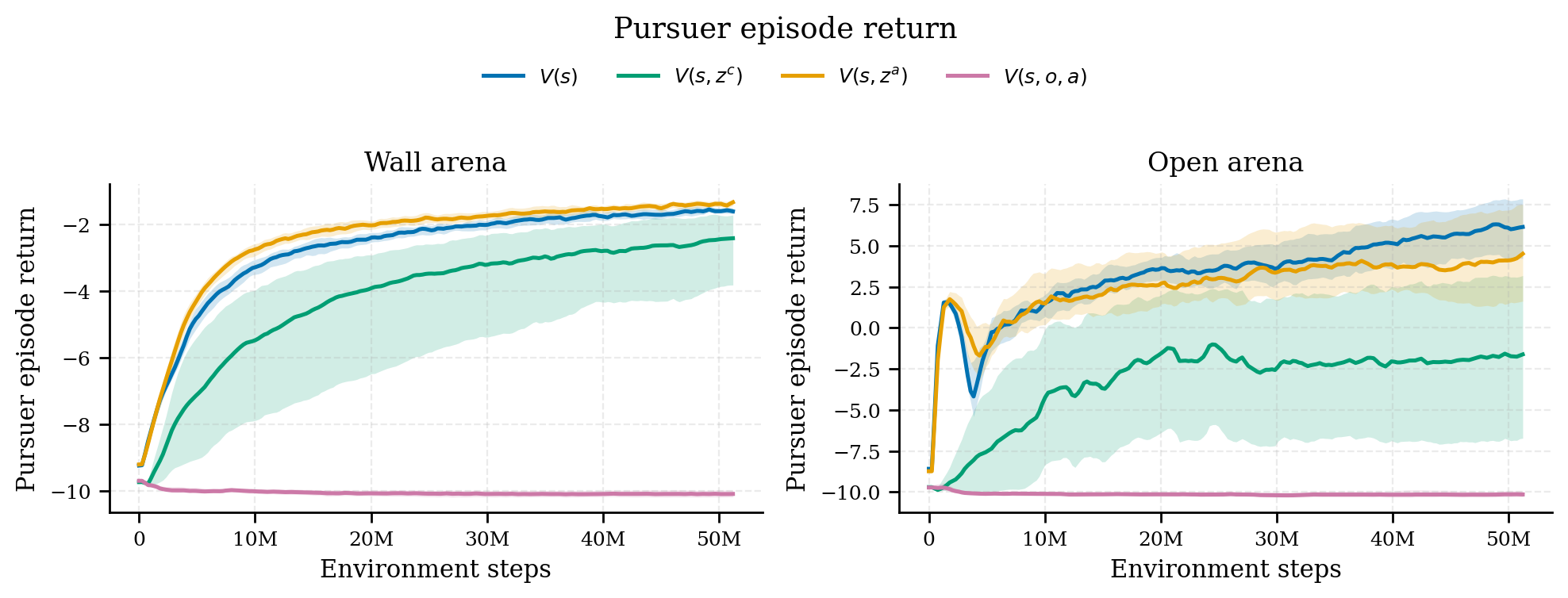}
    \caption{Pursuer episode return during training. Lines are interquartile means over 20 seeds \citep{agarwal2022deepreinforcementlearningedge}. The shaded intervals are $95\%$ bootstrap confidence intervals. The memory-state critic matches or exceeds the history-state critic in both arenas and converges faster. It exceeds the state-only critic in the wall arena and matches it in the open arena.}
    \label{fig:reward_evolution}
\end{figure}

\textbf{The memory-state critic outperforms the history-state critic and converges faster.}
Figure~\ref{fig:reward_evolution} reports the episodic return during training.
In both arenas, the memory-state critic reaches at least the final return of the history-state critic in fewer steps.
This is the outcome demonstrated by Proposition~\ref{prop:equality}.
In the open arena, the history-state critic plateaus lower and has way more variance across seeds.
We attribute this to its second optimisation problem, since its encoder must learn from the value loss alone a representation predictive of future return.
Therefore, it should be predictive of the actor's next action without observing the actor's own representation, which is what the memory-state critic reuses for
free. 
While these theoretical and empirical findings are promising, they are currently limited in terms of environment diversity.
Because the construction holds for any asymmetric POMDP with a deterministic recurrent actor, comparing it on standard POMDP benchmarks is a natural next step.

\textbf{The bias of the state-only critic matters only when the state aliases the history.}
Propositions~\ref{prop:equality} and~\ref{prop:pg} guarantee that the critics are unbiased, but not that they beat the biased state-only critic in every environment.
The size of the bias scales with how much the privileged state aliases the history~\citep{lambrechts2025theoreticaljustificationasymmetricactorcritic}.
In the open arena, the evader is visible whenever it is in frame, and the observation model is deterministic, so the state is close to Markov for the future return, and the state-only critic is nearly unbiased.
In addition, the state-only critic has an easier architecture to train.
Its return is therefore comparable to the memory-state critic's, consistent with it solving a smaller learning problem.
In the wall arena, the evader is repeatedly occluded, so two histories that reach the same state can disagree on where the evader was last seen.
The state-only critic cannot represent this difference, in contrast to the history-state and memory-state critic.
This could explain why the memory-state critic marginally performs better in the wall arena and the history-state critic is closer to the other in episode return

\textbf{The reactive critic did not learn in our environment.}
The reactive critic $V_\psi(s, o, a)$ replaces the recurrent memory with the current observation and previous action, and fails to learn a useful policy in either arena (Figure~\ref{fig:reward_evolution}).
The current observation and previous action carry far less information about the return than the history.
This confirms that the memory term, and not merely the state, is what recovers the history-state performance.

\begin{figure}[h]
    \centering
    \includegraphics[width=0.95\linewidth]{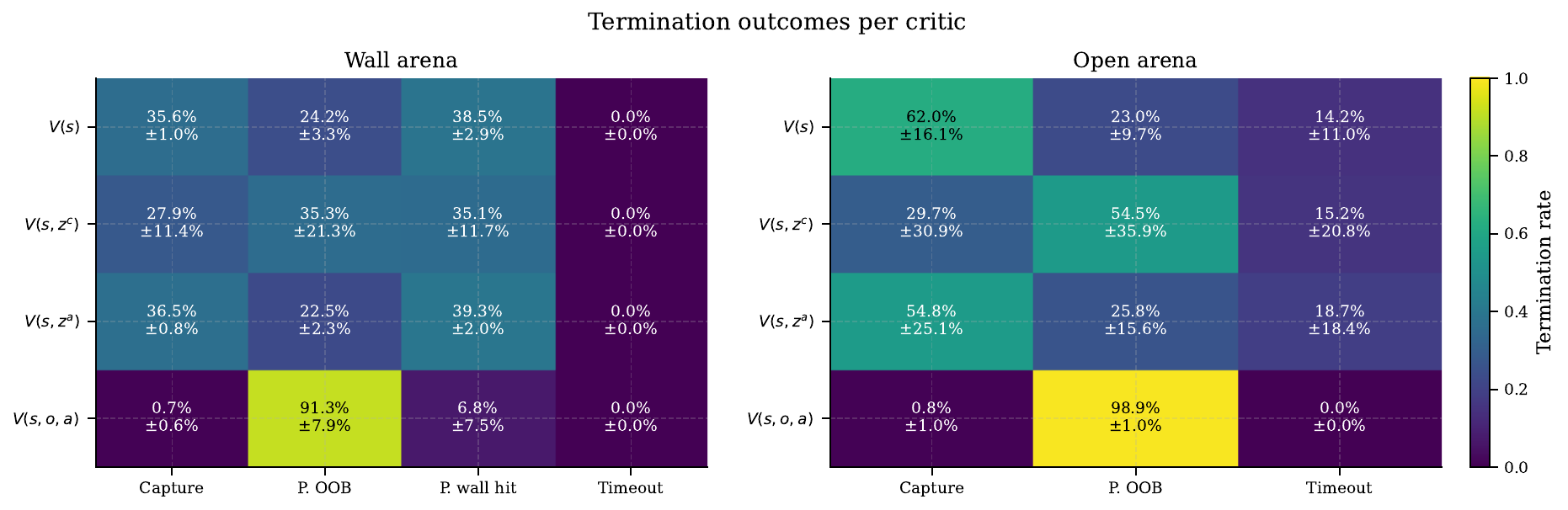}
    \caption{Terminal-condition breakdown per critic and arena. Each cell shows the fraction of evaluation episodes ending in the given outcome as the mean over 20 seeds, $\pm$ one standard deviation. The wall arena adds the pursuer wall collision outcome.}
    \label{fig:terminations}
\end{figure}

\textbf{Termination analysis.}
Figure~\ref{fig:terminations} decomposes how episodes end.
Among the critics that learns to capture the evaders, the history-state critic has the highest pursuer out-of-bounds (OOB) rate in the open arena, consistent with its lower return, as the actor overshoots the evader rather than tracking it.
The reactive critic $V_\psi(s, o, a)$ crashes (equivalent to OOB) in almost every episode, but it never learns to capture.
The memory-state critic achieves the best capture rate in the wall arena, where occlusion makes the privileged state alias the history.
In the wall arena, pursuer wall collisions account for roughly $35$ to $40\%$ of episodes for all learning critics, indicating that wall avoidance is a hard subproblem independent of the critic's history representation.

\section{Related work}
\label{sec:related}

\paragraph{Asymmetric actor-critic.}
\citet{pinto2018asymmetric} introduced the asymmetric actor-critic algorithm for image-based robot learning, with a reactive policy and a state-only critic.
\citet{baisero2022unbiased} prove that with a history-dependent policy, the time-invariant state value function is generally ill-defined, and that even when well-defined the resulting policy gradient is biased.
They propose the history-state critic and show that it yields an unbiased policy gradient.
\citet{lyu2022deeper} extend the analysis to multi-agent reinforcement learning. 
\citet{lambrechts2025theoreticaljustificationasymmetricactorcritic} give a finite-time convergence analysis with linear function approximators.
\citet{ebi2025informed} generalise the construction to partial privileged information beyond the full state. 
Our work provides a deep RL implementation of the history-state critic in which the policy's recurrent encoder is reused on the critic side, removing the second encoder.

\paragraph{Recurrent policies for POMDPs.}
Recurrent policies are a standard way to handle partial observability, and are competitive with transformer-based variants in RL \citet{ni2022recurrentmodelfreerlstrong, hausknecht2015deep}. The history-dependent policy in our work follows this design. Our contribution is on the critic side.

\paragraph{Vision-based drone control.}
\citet{kaufmann2023champion} train a vision-based drone racing policy in simulation with an asymmetric actor-critic, with the policy relying on state estimation from off-board sensors at training time.
\citet{geles2024demonstrating} remove the state-estimation step and train an end-to-end vision policy with a reactive actor.
The questions addressed in this work do not arise in their setting because their policy is non-recurrent. Drone racing also exhibits little aliasing, because the next gate rarely leaves the camera frame, in contrast to pursuit-evasion.
Model-based vision RL has shown strong results in single-agent racing \citep{romero2026dream, verraest2025skydreamer} and is complementary to our work.

\paragraph{Quadrotor pursuit-evasion.}
\citet{roncero2025learned} and \cite{hansen2026surveillance} train quadrotor pursuit-evasion policies under full state observation in a multi-agent setting.
\citet{bajcsy2023learning} consider vision-based pursuit-evasion with a ground robot evader. 
To the best of our knowledge, our environment is the first open-source vision-based quadrotor pursuit-evasion with propeller-level dynamics.

\section{Conclusion}
\label{sec:conclusion}
We introduced the memory-state critic, a single-encoder asymmetric critic for recurrent actors in POMDPs. 
It uses the actor's recurrent hidden state, which avoids learning a second statistic of the history for the value function.
We proved that it gives the same on-policy value function and unbiased policy gradient as the history-state critic of \citet{baisero2022unbiased}, with one recurrent encoder instead of two. 
In a vision-based pursuit-evasion environment between two quadrotors, the memory-state critic outperforms the history-state critic in final performance and converges faster.
It also marginally outperforms the state-only critic, especially when partial observability is important.
Sim-to-real transfer of vision-based drone policies is an active area \citep{yu2025depth}, and physical deployment is future work.

\begin{ack}
Gaspard Lambrechts is a \emph{Marie Sk\l{}odowska-Curie Actions} (MSCA) postdoctoral fellow at McGill University in Canada, funded by \emph{Horizon Europe} under grant 101284668.
Gaspard Lambrechts also gratefully acknowledges the financial support of \emph{Wallonia-Brussels International} for his \emph{World Excellence Fellowship}.
\end{ack}


\bibliographystyle{plainnat}   
\bibliography{refs}

\appendix

\section{Appendix}
\label{app:appendix}

\paragraph{Quadrotor dynamics.}
We model each quadrotor as a rigid body of mass $m$ and inertia tensor $\mathbf{J}$, with state $(p, v, R, \omega, \Omega)$ for position, velocity, attitude, body rates, and the four rotor speeds. 
The continuous dynamics follow \citet{kaufmann2023champion}. Each rotor produces thrust $k_T \Omega_i^2$ and moment $k_M \Omega_i^2$, and the rotor speeds follow first-order dynamics $\dot\Omega_i = k_\text{mot} (\Omega_{\text{ss}, i} - \Omega_i)$ towards the commanded steady-state speed $\Omega_{\text{ss}, i}$.

\paragraph{Action space and low-level control.}
The policy outputs a four-dimensional action $a \in [-1, 1]^4$ at $25$ Hz, three commanded body rates and a collective thrust.
Body rates are scaled to $\omega^{\text{ref}}_k = \pi \, \omega^{\text{cmd}}_k$ rad/s, and the thrust to $T^{\text{ref}} = \tfrac{1}{2}(T^{\text{cmd}} + 1) \, m g \, \eta_{\text{T2W}}$ with $\eta_{\text{T2W}} \approx 3.15$. An inner PD controller at $500$ Hz tracks these references and produces the per-rotor setpoints.

\paragraph{Camera.}
Each drone carries a forward-facing pinhole camera rendered with IsaacLab's tiled-rendering pipeline, at $64 \times 64$ resolution and a $120^\circ$ horizontal field of view. We read two channels, depth and a binary opponent-segmentation mask, with a clipping range $0.02$ to $20$ m.

\paragraph{Reward and arena.}
The reward coefficients in (\ref{eq:reward}) are $R = 10.0$, $\kappa_a = 3.0$, $d_0 = 0.5$ m, $\kappa_p = 0.02$, $\kappa_\rho = 1.0$, and $\kappa_{\text{br}} = 0.001$, with capture distance $r_c = 0.2$ m and maximum episode length $T = 250$ steps. The arena bounds are $[-2.5, 2.5] \times [-2.0, 2.0] \times [0.0, 2.0]$ m. The wall in the occluded arena is $0.05$ m thick with a $1.5$ m gap, and the drone-wall clearance is $0.15$ m. The body-rate weighting is $W_{\text{br}} = \mathrm{diag}(1, 1, 0.2)$, which lowers the yaw penalty to allow reorientation.

\begin{table}[h]
\centering
\caption{PPO hyperparameters, shared across all critic variants.}
\label{tab:hyperparams}
\begin{tabular}{ll}
\toprule
Parameter & Value \\
\midrule
Discount factor $\gamma$ & 0.99 (open) / 0.999 (wall) \\
GAE $\lambda$ & 0.95 \\
Learning rate & $3 \times 10^{-4}$ \\
LR scheduler & KL-adaptive ($\tau_{\text{KL}} = 0.008$) \\
Rollout length & 32 steps \\
Mini-batches & 4 \\
Learning epochs & 8 \\
Clip ratio $\epsilon$ & 0.2 \\
Value clip & 0.2 \\
Entropy coefficient & 0.01 \\
Value loss coefficient & 2.0 \\
Gradient norm clip & 1.0 \\
Reward scaling & 0.1 \\
Initial log $\sigma$ & $0.0$ \\
\midrule
GRU hidden size & 256 \\
GRU layers & 1 \\
Backpropagation sequence length & 16 \\
\midrule
Parallel environments & 512 \\
Policy rate & 25 Hz \\
Simulation rate & 500 Hz \\
Episode length & 10 s \\
\bottomrule
\end{tabular}
\end{table}

\end{document}